\documentclass[10pt,twocolumn,letterpaper]{article}

\usepackage[pagenumbers]{wacv}           

\usepackage{amsthm}            
\usepackage{multirow}          
\usepackage{algorithm}         
\usepackage{algpseudocode}     

\definecolor{wacvblue}{rgb}{0.21,0.49,0.74} \usepackage[pagebackref,breaklinks,colorlinks,allcolors=wacvblue]{hyperref}

\AtEndPreamble{%
  \crefname{equation}{Eq.}{Eqs.}%
  \Crefname{equation}{Equation}{Equations}%
}

\newtheorem{theorem}{Theorem} 

\def\wacvPaperID{1787}
\def\confName{WACV}
\def\confYear{2027}

\begin{document}

\title{CANAL: Channel-Aware Noise Allocation for Differentially Private \\
Feature Distillation in Medical Image Segmentation}

\author{Armaghan Butt \quad Shuya Feng \quad Qing Tian\\
Department of Computer Science, University of Alabama at Birmingham\\
Birmingham, AL, USA\\
{\tt\small \{ab36, fengs, qtian\}@uab.edu} }
\maketitle

\begin{abstract}
Medical image segmentation needs diverse training data, but hospitals hold complementary scans they cannot share for privacy and regulatory reasons. Knowledge distillation can bridge this gap by exporting learned feature representations instead of images, but those representations still encode patient-specific anatomy and remain vulnerable to membership-inference and feature-inversion attacks. Adding calibrated Gaussian noise restores a differential-privacy guarantee, yet three issues have been overlooked. First, prior DP feature-distillation pipelines re-sample noise at every student iteration, so each patient image is released many times and the privacy cost composes over those releases, growing by orders of magnitude. We present a sample-once-per-image release, realized by a single precomputation pass, under which each patient contributes one release.
Second, uniform noise is wasteful because channels differ in task importance. Using task-gradient energy as the importance measure, we derive CANAL, a closed-form water-filling allocation that gives important channels proportionally less noise, and prove it strictly minimizes importance-weighted distortion at a fixed budget.
Third, the clipping caps and importance scores that drive the allocation are themselves data-dependent, so releasing them in the clear silently breaks the guarantee. We give a DP-honest budget split that charges each to the privacy budget, so the reported $\varepsilon$ is the true $\varepsilon$. Across three medical segmentation benchmarks spanning dermoscopy, colonoscopy, and ultrasound, CANAL retains more task-relevant signal than uniform noise at the same privacy budget.
\end{abstract}

\section{Introduction}
\label{sec:intro}

\begin{figure*}[t]
  \centering
  \includegraphics[width=\textwidth]{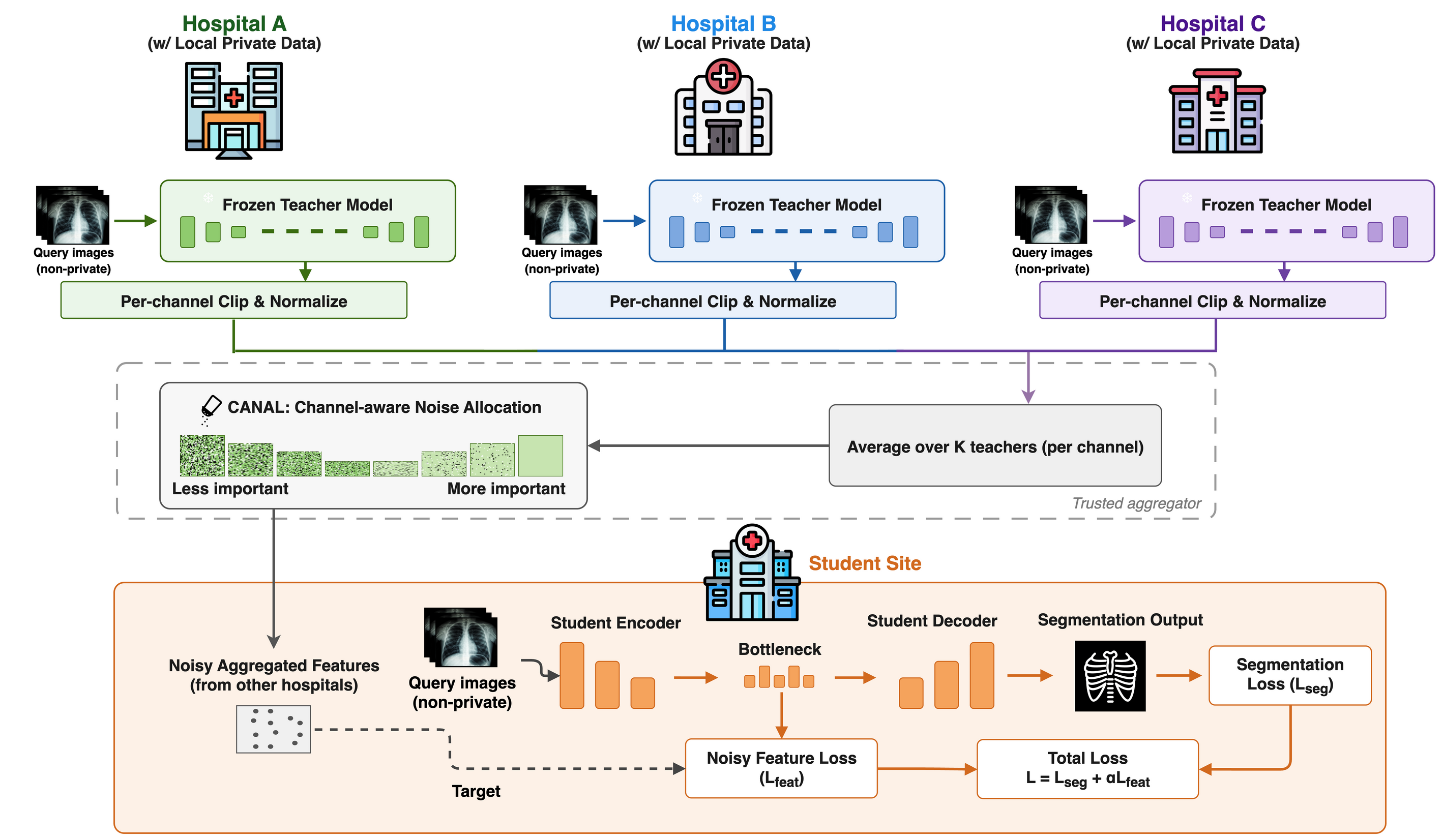}
  \caption{\textbf{CANAL multi-hospital private feature distillation.} Each hospital trains a teacher on its private scans. Shared non-private query images are passed once through each frozen teacher, and each teacher's bottleneck is per-channel clipped and normalized. The per-channel average across the $K$ teachers is then perturbed once by CANAL's channel-aware noise, with less important channels (by task-gradient energy) receiving more noise. The resulting noisy aggregate is the distillation target that the student, trained on the same query images, matches with a feature loss ($L_{\mathrm{feat}}$) alongside its segmentation loss ($L_{\mathrm{seg}}$). The dashed region marks the trusted aggregator, which forms the average and adds the noise before anything is released (\cref{ssec:impl}). Raw scans never leave a site, and each image's features are released once under an $(\varepsilon,\delta)$-DP guarantee.}
  \label{fig:teaser}
\end{figure*}

A segmentation model trained at a single hospital is limited by the size and diversity of that hospital's data, which reflects one set of scanners, acquisition protocols, and patient population. Training across many sites would address this, but the most direct route, pooling everyone's scans into one dataset, is blocked by patient-privacy regulation and institutional policy~\cite{kaissis2020secure}. The central question for trustworthy medical AI is therefore how one hospital can benefit from others' data without exposing any individual patient at the data-providing sites. We assume the recipient hospital holds its own non-private data on which the shared knowledge is distilled, so that privacy need only be preserved for the data providers.

Knowledge distillation~\cite{hinton2015distilling,romero2015fitnets} suggests a natural division of labor. A site trains a teacher model locally on its own scans. Rather than shipping raw images, it exports only learned signals, the teacher's logits or intermediate feature maps, and a model at another site is trained to match them. This is attractive because the heavy, identity-bearing object, the image, never leaves the hospital. The difficulty is that the exported signals are not as innocuous as they appear. A segmentation network's intermediate feature maps encode patient-specific anatomical and pathological characteristics, and a growing body of work shows that representations shared in this way are vulnerable to membership-inference attacks, which decide whether a particular patient was in the training set~\cite{shokri2017membership}, and to feature-inversion attacks, which reconstruct recognizable input structure from the released features~\cite{dosovitskiy2016inverting}. Sharing features is not, by itself, sharing privately.

Differential privacy (DP)~\cite{dwork2006calibrating} provides the remedy that membership and inversion attacks lack a defense against: calibrated Gaussian noise added to each released feature representation provably bounds how much any one patient can influence what is shared, and so bounds what any downstream adversary can infer. The strength of this guarantee is captured by a single privacy-loss parameter $\varepsilon$: smaller $\varepsilon$ means stronger privacy (more noise), with single-digit $\varepsilon$ the usual operating range.
The resulting recipe is simple: each site trains a teacher locally and privatizes its features and logits with noise, and any site learns a student from the privatized signals. This delivers the multi-site benefit while keeping a per-patient guarantee.
Our work is about making this recipe both honest and efficient for medical image segmentation, and \cref{fig:teaser} overviews the resulting pipeline.

\textbf{Composition-free release.} Differentially private feature distillation re-noises the teacher's released features at every student iteration~\cite{wu2025multiview}, so each patient's data is queried numerous times and the privacy cost composes over those queries, an accounting pressure that also arises when private distillation transfers teacher labels rather than features~\cite{li2022finegrained}. Holding a target user-level $\varepsilon$ then forces the budget to be split across all the releases, so the per-release noise grows with their number and overwhelms the signal at single-digit $\varepsilon$. We argue the correct model for this setting is sample-once-per-image: each patient image is pushed through the teacher once, noised once, and cached, so that every patient contributes exactly one release and the entire budget funds it---the user-level $\varepsilon$ equals the per-release $\varepsilon$. A single precomputation pass realizes this faithfully.

\textbf{Efficient allocation.} Given a fixed per-release budget, adding the same amount of noise to every channel is wasteful, because channels are not equally important to the task. In this paper, we measure channel importance by the task-gradient energy, and pose noise allocation as minimizing importance-weighted distortion subject to the privacy constraint. The solution, which we call CANAL (Channel-Aware Noise ALlocation), is a closed-form channel-importance water-filling (WF) rule, in the sense of the classic resource-allocation principle of spending a fixed budget where it helps most: more important channels receive proportionally less noise, and the allocation provably beats uniform noise whenever channel importance is non-uniform. Gradient-based importance has appeared in distillation as a way to re-weight a training loss~\cite{lan2024gkd}. Allocating a hard privacy budget across independently noised channels is a different problem (where the matched measure is the gradient energy, not the signed gradient), and we solve it optimally in closed form rather than heuristically. Because the clipping caps and importance scores are themselves computed from data, we also show how to charge each to the budget in a DP-honest split.

\textbf{Contributions.} (i) We identify the per-iteration composition cost in DP feature distillation and present a sample-once-per-image release that pays the privacy budget once, so the user-level $\varepsilon$ equals the per-release $\varepsilon$ (\cref{ssec:release,ssec:accounting}). (ii) We derive CANAL, a closed-form channel-importance water-filling allocation, and prove it strictly minimizes importance-weighted distortion at a fixed budget (\cref{thm:wf}). (iii) We observe that the clipping caps and importance scores driving the allocation are data-dependent, so releasing them in the clear, as prior work implicitly does, leaks. We give a DP-honest budget split that charges each to the budget, so the reported $\varepsilon$ holds (\cref{ssec:budget-split}). (iv) We validate the pipeline on three medical segmentation benchmarks across distinct modalities (dermoscopy, colonoscopy, ultrasound) (\cref{sec:experiments}).

\section{Related Work}
\label{sec:related}

\paragraph{Knowledge and feature distillation.} Knowledge distillation transfers a teacher's behavior to a student, originally by matching soft output logits~\cite{hinton2015distilling}. Feature distillation is the variant that matches the teacher's intermediate feature maps instead, so the student mimics its internal representations rather than only its outputs~\cite{romero2015fitnets}. Gradient-Guided Knowledge Distillation (GKD)~\cite{lan2024gkd} weights each feature channel by the spatially average-pooled task-loss gradient, emphasizing the channels that matter most for the downstream objective. We adapt this gradient-based notion of channel importance: because our mechanism adds independent per-channel noise, the matched quantity is the gradient energy rather than the pooled signed gradient, and we use it to allocate a privacy budget across channels (\cref{eq:gkd-importance}).

\paragraph{Privacy attacks on shared representations.} The need for noise on shared features is established empirically by attacks. In vision, feature-inversion networks reconstruct recognizable images directly from intermediate representations~\cite{dosovitskiy2016inverting}, and gradient-leakage attacks recover entire training images, pixel for pixel, from shared gradients~\cite{zhu2019deep,yin2021gradinversion}, precisely the kind of signal a distillation pipeline exports. Membership-inference attacks further decide whether a specific patient was in the training set~\cite{shokri2017membership}, and classical model inversion reconstructs input attributes from a trained model~\cite{fredrikson2015model}. In medical imaging, where the input is the patient, all of these translate directly into re-identification risk~\cite{kaissis2020secure}, motivating a formal guarantee rather than ad-hoc obfuscation.

\paragraph{Differentially private learning.} DP~\cite{dwork2006calibrating} bounds the influence of any single record on a released quantity. The standard way to train a model under DP is DP-SGD~\cite{abadi2016deep}, which clips and noises per-example gradients. PATE~\cite{papernot2017semi,papernot2018scalable} instead trains an ensemble of teachers on disjoint private partitions and transfers a noisy aggregate of their votes to a student, an architecture close in spirit to the multi-hospital setting we target. In medical imaging, federated learning~\cite{mcmahan2017communication} keeps raw scans on site but still leaks through shared updates unless combined with DP (now a standard recommendation~\cite{kaissis2020secure}), while input-space methods such as DP-Pix~\cite{fan2018image} privatize the pixels directly, trading substantial utility for the guarantee. These mechanisms noise different objects: gradients (DP-SGD), label votes (PATE), model updates (federated DP), or raw pixels (DP-Pix). We instead noise the released features once, before any student training, and spend the budget non-uniformly across channels to retain more task-relevant signal at the same $\varepsilon$.

\section{Private Feature Release and Channel-Importance Noise Allocation} \label{sec:privacy}

We formalize the mechanism as follows. \Cref{ssec:release} states the sample-once release model and explains why it differs from the implicit per-iteration model used by prior DP feature distillation work. \Cref{ssec:accounting} gives the composition accounting. \Cref{ssec:wf} derives the channel-importance water-filling allocation across the bottleneck channels. \Cref{ssec:budget-split} charges the data-dependent clipping caps and importance scores to the budget in a DP-honest split, \cref{ssec:granularity} explains why we allocate across channels rather than pixels, and \cref{ssec:impl} describes the single-pass implementation.

\subsection{Release model: sample-once-per-image}
\label{ssec:release}

We consider a federated medical setting in which a data-holding site processes each patient image once through the teacher encoder, adds Gaussian noise to the resulting bottleneck representation, and releases the noisy representation. A separate party trains the student against these fixed noisy targets. Formally, for training set $\mathcal{D}=\{x_i\}_{i=1}^{N}$ and teacher encoder $\phi_T : \mathcal{X}\!\to\!\mathbb{R}^{C\times H\times W}$, the released objects are
\begin{equation}
\begin{aligned}
\widetilde{z}_i &\;=\; \mathrm{denorm}\!\bigl(\,\mathrm{norm}(\phi_T(x_i)) + \eta_i\,\bigr),\\
\eta_i &\sim \mathcal{N}(0,\,\mathrm{diag}(\sigma^2)),
\end{aligned}
\label{eq:release}
\end{equation}
where $\mathrm{norm}(\cdot)$ performs per-channel $L_2$ clipping and normalization so that each channel of the released tensor lies in the unit ball before noise is added.

\textbf{Why this differs from per-iteration release.} Differentially private feature distillation re-noises the teacher's released features at every training iteration~\cite{wu2025multiview}, so each training image's features are released $Q = TB/N$ times across $T$ iterations of batch size $B$. Composition over those $Q$ releases multiplies the user-level $\varepsilon$, the same per-query accounting faced by private distillation that transfers teacher labels rather than features~\cite{li2022finegrained}.
The sample-once model in \cref{eq:release} eliminates this composition entirely: each user contributes exactly one release.

\subsection{Composition and budget accounting}
\label{ssec:accounting}

We use zero-Concentrated Differential Privacy (zCDP) \cite{bun2016concentrated} as the working unit: a variant of DP whose Gaussian-mechanism costs add up cleanly under composition, summarized by a single number $\rho$ that we convert to the usual $(\varepsilon,\delta)$ at the end. For a Gaussian mechanism on a $C$-dimensional vector with per-coordinate sensitivity $\Delta_c$ (the most that replacing one patient can move that coordinate) and per-coordinate noise scale $\sigma_c$, the per-release zCDP cost is
\begin{equation}
\rho_{\mathrm{rel}}
 \;=\; \sum_{c=1}^{C}\frac{\Delta_c^2}{2\,\sigma_c^2}.
\label{eq:rho-rel}
\end{equation}
After per-channel $L_2$ normalization, each channel has norm at most $1$, so the replace-one $L_2$ sensitivity is the data-independent constant $\Delta_c = 2/K$ for all $c$, where $K\geq 1$ is the PATE-style teacher ensemble size.

\paragraph{Conversion to $(\varepsilon,\delta)$-DP.} By \cite[Prop.~1.3]{bun2016concentrated}, a $\rho$-zCDP mechanism is $(\rho + 2\sqrt{\rho\log(1/\delta)},\,\delta)$-DP for any $\delta>0$. Conversely, the inverse map
\begin{equation}
\rho(\varepsilon,\delta)
 \;=\;\Bigl(\sqrt{\log(1/\delta)+\varepsilon}-\sqrt{\log(1/\delta)}\,\Bigr)^{2}
\label{eq:eps-to-rho}
\end{equation}
gives the maximum zCDP that fits inside a target $(\varepsilon,\delta)$-DP budget.

\paragraph{Composition.} The two release models differ only in the number of times each user's data participates in a release.
\begin{itemize}
\item \textbf{Sample-once per image.} Each user contributes a single
release. The user-level cost is $\rho_{\mathrm{user}} = \rho_{\mathrm{rel}}$, giving user-level $\varepsilon_{\mathrm{user}} = \varepsilon_{\mathrm{rel}}$.
\item \textbf{Per-iteration release.} Each user participates in
$Q = TB/N$ releases of fresh noise on fresh teacher queries. Naive zCDP composition gives $\rho_{\mathrm{user}} = Q\,\rho_{\mathrm{rel}}$, which is then mapped back to $(\varepsilon_{\mathrm{user}}, \delta)$ via \cite[Prop.~1.3]{bun2016concentrated}.
\end{itemize}

\paragraph{The composition penalty.} Under the per-iteration release model, each training image is released $Q=TB/N$ times over a training run, so the cost composes and can multiply the user-level $\varepsilon$ by orders of magnitude. The sample-once model in \cref{eq:release} pays the budget once, so the user-level $\varepsilon$ equals the per-release $\varepsilon$ by construction.

\subsection{Channel-importance water-filling allocation}
\label{ssec:wf}

Because \cref{eq:rho-rel} constrains only the sum of the per-channel costs, a fixed per-release budget $\rho_{\mathrm{rel}}$ leaves the individual noise scales $\{\sigma_c\}$ free: many allocations meet the same budget. We pick the allocation that minimizes the importance-weighted distortion
\begin{equation}
\min_{\sigma_1,\ldots,\sigma_C > 0}
\sum_{c=1}^C s_c\,\sigma_c^{2}
\quad\text{s.t.}\quad
\sum_{c=1}^C\frac{\Delta_c^{2}}{2\,\sigma_c^{2}} \le \rho_{\mathrm{rel}},
\label{eq:opt}
\end{equation}
where $s_c \ge 0$ is a per-channel importance score, defined next. This is a rate--distortion-style allocation, familiar from bit allocation in image compression: a fixed budget, here privacy rather than bits, is spent so as to minimize a weighted distortion, placing more noise where it costs the task least. Solving \cref{eq:opt} via Lagrangian gives the closed-form CANAL allocation.

\paragraph{Gradient-energy channel importance.} A channel matters to the segmentation task in proportion to how much perturbing it changes the task loss, and the loss gradient can measure this sensitivity directly. The mechanism of \cref{eq:release} adds independent Gaussian noise at every spatial location, so the quantity a privacy mechanism should consult is not the net directional gradient but its energy: for noise of variance $\sigma_c^2$ in channel $c$, the expected squared change in the task loss is $\sigma_c^2\sum_{h,w}\bigl(\partial\mathcal{L}_{\mathrm{task}}/\partial A_i[c,h,w]\bigr)^2$, so the per-channel weight under which the water-filling objective \cref{eq:opt} equals the true expected distortion is the squared-gradient energy. For a training image $x_i$, with $A_i = \phi_T(x_i) \in \mathbb{R}^{C\times H\times W}$ the teacher bottleneck activation and $\mathcal{L}_{\mathrm{task}}$ the supervised segmentation loss, we therefore set
\begin{equation}
g_{i,c} \;=\;
\frac{1}{HW}\sum_{h=1}^{H}\sum_{w=1}^{W}
\Bigl(\frac{\partial \mathcal{L}_{\mathrm{task}}}{\partial A_i[c,h,w]}\Bigr)^{2},
\label{eq:gkd-importance}
\end{equation}
the per-location average of the squared task gradient. The common factor $1/HW$ rescales every channel equally and so leaves the allocation unchanged. Averaging over the training set gives the importance score $s_c = \tfrac{1}{N}\sum_{i=1}^{N} g_{i,c} \ge 0$. Channels whose perturbation strongly changes the segmentation output receive a large $s_c$ and, by \cref{eq:wf}, are protected with proportionally less noise. The double aggregation in \cref{eq:gkd-importance}, spatial pooling followed by averaging over images, also keeps $s_c$ a low-dimensional, dataset-level statistic that is cheap to privatize (\cref{ssec:granularity}). The gradients are computed once on the teacher during the same precompute pass that produces the noisy release.

\begin{theorem}[Channel-WF optimal allocation]
\label{thm:wf}
For positive importance scores $\{s_c\}$, positive sensitivities $\{\Delta_c\}$ and budget $\rho_{\mathrm{rel}}>0$, the unique minimizer of \cref{eq:opt} is
\begin{equation}
\sigma^{\star}_c \;=\; \kappa \,\sqrt{\Delta_c}\, s_c^{-1/4},
\qquad
\kappa \;=\; \sqrt{\frac{1}{2\rho_{\mathrm{rel}}}
                          \sum_{c=1}^{C}\Delta_c\,\sqrt{s_c}}.
\label{eq:wf}
\end{equation}
The minimized objective scales as $\bigl(\sum_c \Delta_c\sqrt{s_c}\bigr)^2$ and is strictly smaller than that of the uniform allocation $\sigma^{\mathrm{unif}} = \sqrt{(\sum_c \Delta_c^2) / (2\rho_{\mathrm{rel}})}$ whenever the $s_c$ are non-constant.
\end{theorem}

\begin{proof}[Proof sketch]
Set $u_c = 1/\sigma_c^2$. The constraint becomes $\sum_c \Delta_c^2 u_c / 2 \le \rho_{\mathrm{rel}}$ and the objective $\sum_c s_c / u_c$. Stationarity of the Lagrangian gives $u_c \propto \sqrt{s_c}/\Delta_c$, which inverts to \cref{eq:wf}. Strictness follows from Cauchy--Schwarz. The full proof, including the active-constraint argument and the strict-improvement comparison, is given in \cref{app:proof}.
\end{proof}

Theorem~\ref{thm:wf} implements the intended behavior: high-importance channels (large $s_c$) receive smaller $\sigma_c$, and the closed form makes the per-release budget allocation deterministic given the importance scores.

\subsection{Joint budget split (caps, importance, release)}
\label{ssec:budget-split}

Two upstream quantities used in the pipeline are themselves data-dependent: the per-channel clipping caps $\widehat{c}_c$ and the importance scores $\widehat{s}_c$. Although both are read off the non-private query images the student supplies, they are read off the teacher's activations, and the teacher's weights are a function of its private training cohort. The caps and scores therefore depend on the protected patients, just indirectly through the model, and releasing them in the clear would leak. To be DP-honest, we charge a piece of the total budget to each. With total budget $\rho_{\mathrm{tot}}$, we split
\begin{equation}
\rho_{\mathrm{tot}}
 \;=\;
 \rho_{\mathrm{caps}} \;+\; \rho_{\mathrm{imp}} \;+\; \rho_{\mathrm{rel}},
\label{eq:split}
\end{equation}
release the caps and importance scores via Gaussian mechanisms with the corresponding zCDP costs, and feed the (noisy) outputs into the channel-WF allocation \cref{eq:wf}. The split fractions $(f_{\mathrm{caps}},f_{\mathrm{imp}},f_{\mathrm{rel}})$ with $f_{\mathrm{caps}}+f_{\mathrm{imp}}+f_{\mathrm{rel}}=1$ are hyperparameters. We fix them once across all datasets and budgets (\cref{ssec:exp-setup}) rather than tuning per setting.

\subsection{Channel vs.\ spatial granularity}
\label{ssec:granularity}

The water-filling derivation of Theorem~\ref{thm:wf} is agnostic to the index set: replacing $c$ by a joint index $(c,p)$ over channels and pixels leaves the closed form \cref{eq:wf} unchanged, so one might ask why CANAL allocates across channels rather than the finer $C\!\times\!H\!\times\!W$ grid. The obstruction is not the optimization but the accounting of the importance map that drives it: since the noise scales depend on data-derived scores, an honest mechanism must privatize those scores first (the role of $\rho_{\mathrm{imp}}$ in \cref{eq:split}), and the cost of doing so scales with the map's dimension and sensitivity. Channel scores $s_c$ are cheap on both counts: a $C$-dimensional, dataset-level aggregate (\cref{eq:gkd-importance}) of low replace-one sensitivity, privatized once and amortized over every release. They also carry no spatial index, so they cannot reveal where a patient's anatomy lies, only which channels matter on average. A spatial map enjoys neither property. Averaged over the dataset, it is nearly useless for unregistered images such as dermoscopy or ultrasound scans, where anatomy sits at a different location in each patient. Computed per image, it does pinpoint that location, but such a map is determined by a single image, giving it near-maximal sensitivity and making it impractical to privatize. Channel granularity is thus the honest sweet spot.

\subsection{Implementation}
\label{ssec:impl}

The sample-once release model is realized by a single precomputation pass in the federated, multi-teacher setting: $K$ sites each hold their own disjoint private patients and train a local frozen teacher. Before student training begins, the student site supplies a fixed set of non-private query images. Each image is passed once through every teacher, each teacher's bottleneck is per-channel clipped and normalized, and the per-channel average of the teachers' bottlenecks, perturbed by the noise of \cref{eq:release}, is stored to disk keyed by image identifier. Because each patient belongs to a single site and each teacher's clipped contribution carries weight $1/K$ in the aggregate, replacing one record perturbs at most one teacher and shifts the released mean by at most $\Delta_c = 2/K$ (\cref{ssec:accounting}). Adding calibrated Gaussian noise once to this low-sensitivity aggregate therefore suffices for the zCDP guarantee, whereas noising each teacher independently and averaging afterwards would leave $K$ times the noise variance in the released mean at the same budget. The average is formed by a trusted aggregator before any noise is added, the same trust model PATE~\cite{papernot2017semi,papernot2018scalable} assumes for its vote aggregator. What that party holds here is a clipped bottleneck tensor rather than a label vote, so the accounting is unchanged but the material it sees is richer, and replacing it with secure aggregation is a natural extension we do not pursue. The noise at a fixed budget accordingly shrinks as $K$ grows, and setting $K{=}1$ recovers the single-teacher case. Student training then loads these fixed targets at every iteration, eliminating both fresh teacher queries and fresh noise samples from the inner loop. We use $K{=}3$ teachers in our experiments (\cref{sec:experiments}). Pseudocode appears in \cref{app:pseudocode}.

\paragraph{Releasing a channel subset.} We concentrate the release on the most task-important channels rather than noising all $C$. In practice, all $K$ teachers produce full $C$-dimensional bottlenecks, which are per-channel clipped, normalized, and averaged into one shared $C$-dimensional mean. A single channel mask, derived from the shared importance scores averaged across all $K$ teachers, then zeros out the bottom $C{-}M$ channels of that mean. Gaussian noise is added only on the $M$ active channels. The release is therefore still $C$-dimensional but carries signal only in the top-$M$ channels, and CANAL allocates the budget across those $M$ channels as before. Because all teachers contribute to the same shared importance ranking, the active channel set is identical across teachers by construction. Because the ranking uses only the already-privatized importance scores (\cref{ssec:budget-split}), the selection incurs no additional privacy cost, and $M$ is simply a hyperparameter, with $M{=}C$ recovering the full-bottleneck release.

\section{Experiments}
\label{sec:experiments}

After describing the datasets, models, and privacy configuration (\cref{ssec:exp-setup}), we evaluate CANAL against uniform allocation across three medical segmentation benchmarks spanning distinct modalities and a range of privacy budgets (\cref{ssec:exp-generalize}), and illustrate qualitatively how the allocation reshapes the predicted masks at a tight budget (\cref{ssec:exp-qualitative}). We then measure how well the added noise mitigates concrete membership-inference and feature-inversion attacks (\cref{ssec:exp-attacks}), and finally ablate the ensemble size $K$ (\cref{ssec:exp-ablation}).

\subsection{Experimental setup}
\label{ssec:exp-setup}

\paragraph{Datasets.} We evaluate on three medical segmentation benchmarks spanning distinct imaging modalities: ISIC~2018 (RGB dermoscopy)~\cite{isic2018}, Kvasir-SEG (RGB colonoscopy)~\cite{jha2020kvasir}, and BUSI (single-channel grayscale ultrasound)~\cite{aldhabyani2020busi}. Each training image is one ``user'' for the purposes of user-level DP, consistent with the sample-once-per-image model of \cref{ssec:release}.

\paragraph{Models and distillation.} Teacher and student are U-Nets~\cite{ronneberger2015unet}. The teacher is trained locally on a site's images. Distillation matches the student's bottleneck features, through a $1\!\times\!1$ convolutional adapter, to the top $10\%$ of teacher channels by gradient-energy importance (clipped, normalized, noised) via an $L_2$ feature loss, combined with the supervised segmentation loss on the student. The same gradient-energy importance both selects these channels and sets CANAL's per-release budget allocation across them (\cref{ssec:impl}). The released targets are precomputed once and cached as described in \cref{ssec:impl}, so the student reads fixed targets with no teacher query or fresh noise in its inner loop.

\paragraph{Privacy configuration.} We work in zCDP \cite{bun2016concentrated} with $\delta\!=\!10^{-5}$ and report user-level $\varepsilon\!\in\!\{1,2,4,8\}$ via \cref{eq:eps-to-rho}. Sensitivities follow the per-channel $L_2$ normalization of \cref{ssec:accounting} ($\Delta_c\!=\!2/K$, $K\!=\!3$). Importance scores $s_c$ are the task-gradient energies of \cref{ssec:wf}. Unless stated otherwise we use the budget split $(f_{\mathrm{caps}},f_{\mathrm{imp}},f_{\mathrm{rel}})\!=\!(0.10,0.05,0.85)$ of \cref{ssec:budget-split}.

\subsection{Results across datasets and modalities}
\label{ssec:exp-generalize}

We evaluate across all three benchmarks and a range of privacy budgets. For each dataset we train a $K{=}3$ teacher ensemble and compare CANAL against uniform allocation (the same release with the budget spread equally over all channels) at $\varepsilon\in\{1,2,4,8\}$, reporting the Dice coefficient (DSC, in \%) on the held-out validation split over $5$ seeds, with all mechanisms tuned at the same budget. \Cref{tab:exp-generalize} reports the sweep. As we can see, CANAL matches or exceeds uniform on all three datasets and at every budget, with consistent if sometimes small margins, transferring across both anatomical domain (dermoscopy, colonoscopy, and ultrasound) and input modality (RGB and single-channel grayscale).

\begin{table}[t]
\centering
\small
\caption{\textbf{CANAL vs.\ uniform across datasets and $\varepsilon$} ($K{=}3$, Dice (\%), both approaches share the same seed for each case, better method per dataset in \textbf{bold}).}

\label{tab:exp-generalize}
\begin{tabular}{l l c c c c}
\toprule
Dataset & Method & $\varepsilon{=}1$ & $\varepsilon{=}2$ & $\varepsilon{=}4$ & $\varepsilon{=}8$ \\
\midrule
\multirow{2}{*}{Kvasir-SEG}  & uniform & 69.40 & 73.08 & 74.20 & 73.71 \\
                            & CANAL   & \textbf{70.81} & \textbf{73.56} & \textbf{75.22} & \textbf{74.98} \\
\addlinespace[2pt]
\multirow{2}{*}{BUSI} & uniform & 59.77 & 63.00 & 65.55 & 66.20 \\
                            & CANAL   & \textbf{61.42} & \textbf{64.27} & \textbf{66.59} & \textbf{67.01} \\
\addlinespace[2pt]
\multirow{2}{*}{ISIC 2018}       & uniform & 82.75 & 83.75 & 84.58 & 83.94 \\
                            & CANAL   & \textbf{83.14} & \textbf{84.17} & \textbf{84.79} & \textbf{84.52} \\
\bottomrule
\end{tabular}
\end{table}

\subsection{Qualitative results}
\label{ssec:exp-qualitative}

At a tight budget ($\varepsilon{=}2$), the choice of allocation is visible in the predicted masks. \Cref{fig:qualitative} shows one example case per modality. Under uniform noise, the student's predictions in these cases fragment into scattered, spurious regions and miss much of the lesion, bearing little resemblance to the ground truth. CANAL, which spends the same budget but spares the task-important channels, instead recovers a single coherent region localized on the true lesion or polyp and closely following its boundary. This holds across dermoscopy, colonoscopy, and ultrasound, indicating that concentrating the budget on the channels that carry the segmentation signal preserves the structure the student needs, where uniform noise destroys it.

\begin{figure*}[t]
\centering

\setlength{\tabcolsep}{2pt}
\renewcommand{\arraystretch}{1.0}

\newcommand{\qimg}[1]{%
    \includegraphics[width=0.12\linewidth]{#1}
}

\begin{tabular}{@{}cccc@{}}

\toprule

\textbf{Input} &
\textbf{Ground Truth} &
\textbf{Uniform} &
\textbf{CANAL (ours)} \\

\midrule


\qimg{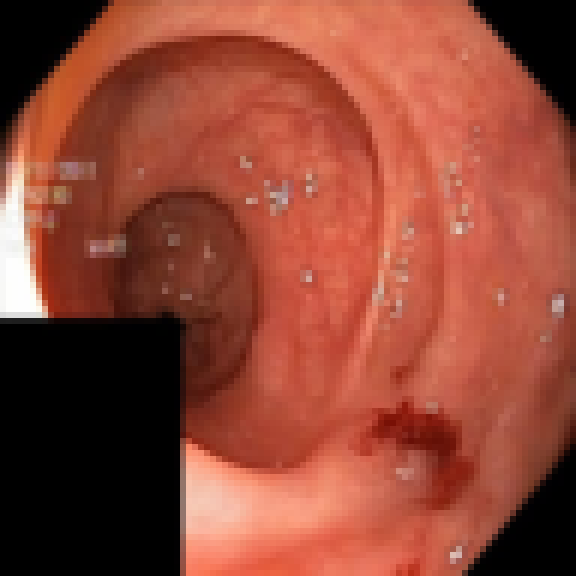} &
\qimg{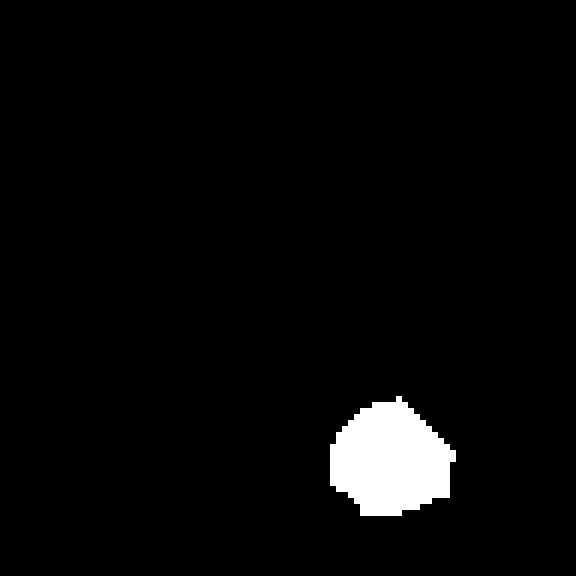} &
\qimg{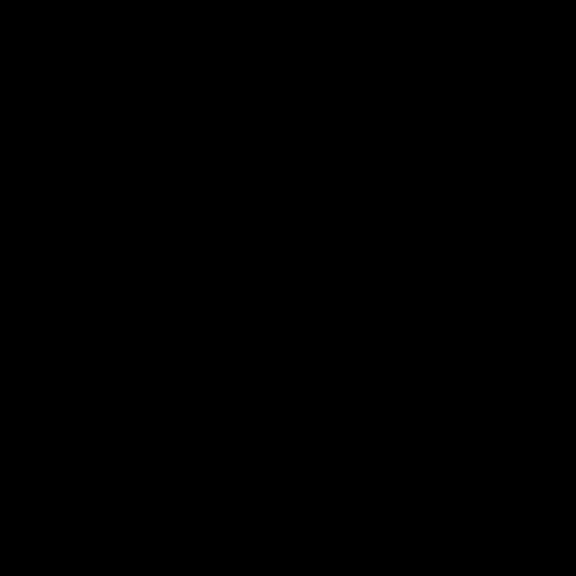} &
\qimg{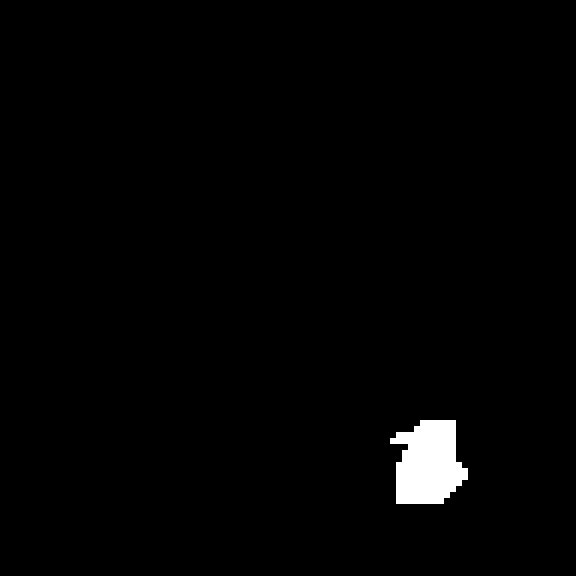} \\

\multicolumn{4}{c}{%
\footnotesize
Kvasir-SEG (Colonoscopy)%
} \\[5pt]


\qimg{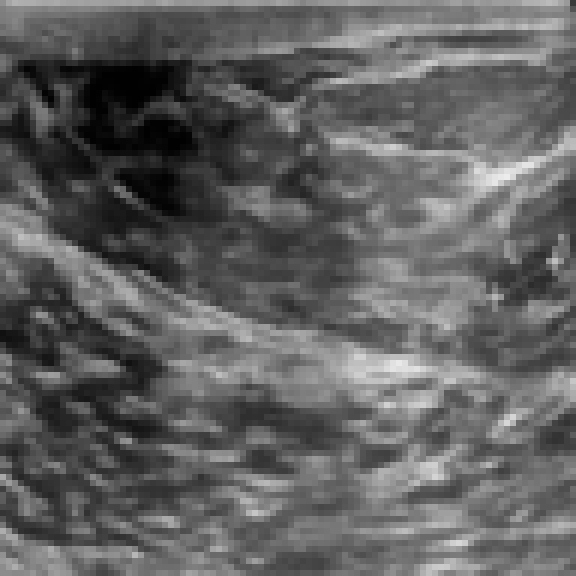} &
\qimg{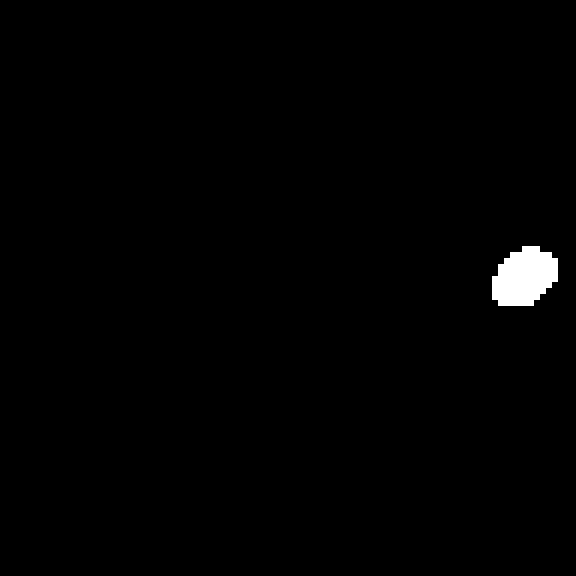} &
\qimg{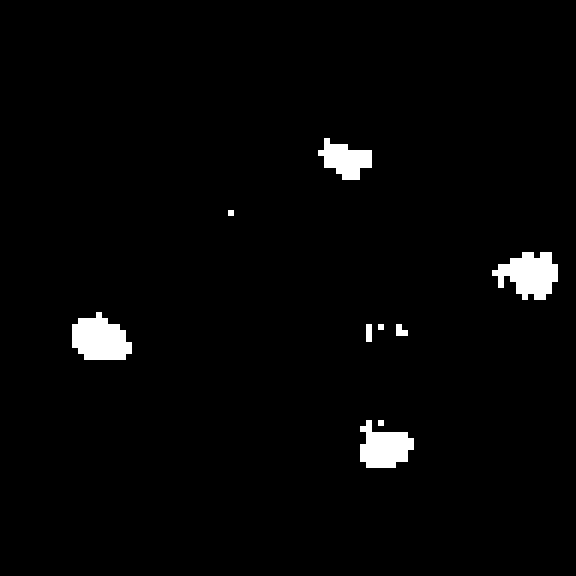} &
\qimg{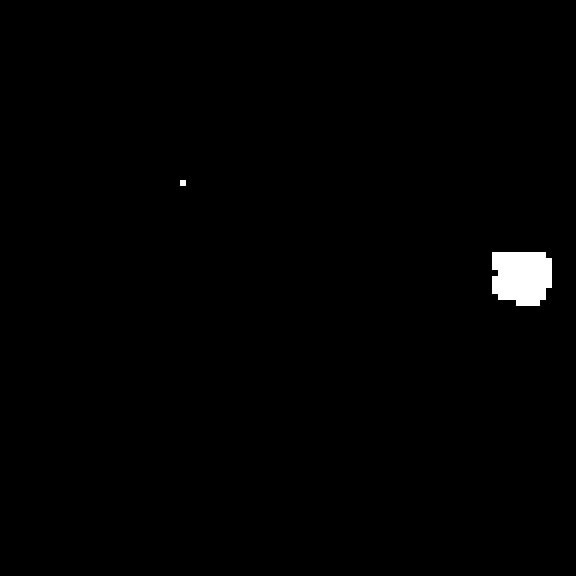} \\

\multicolumn{4}{c}{%
\footnotesize
BUSI (Ultrasound)%
} \\[5pt]


\qimg{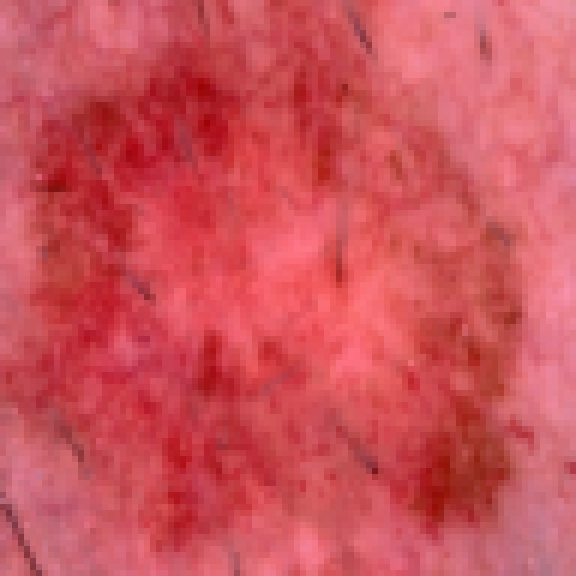} &
\qimg{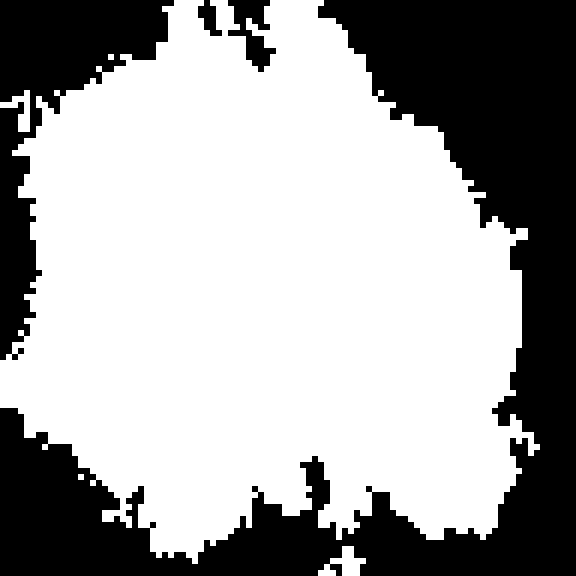} &
\qimg{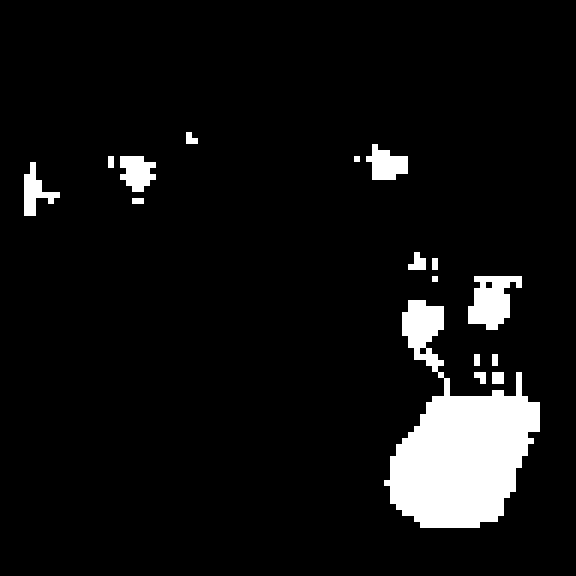} &
\qimg{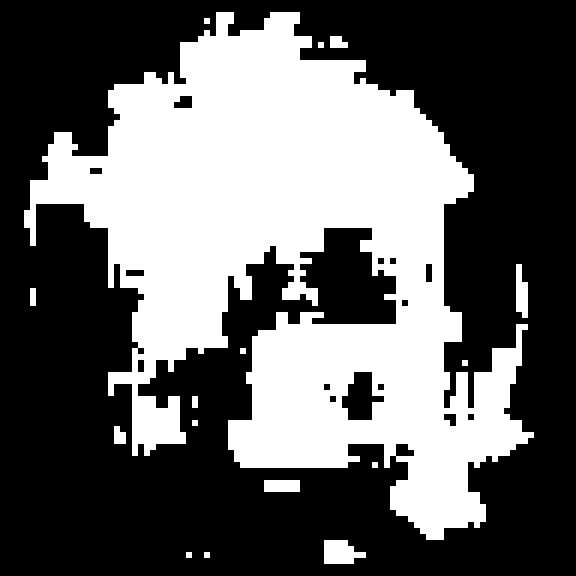} \\

\multicolumn{4}{c}{%
\footnotesize
ISIC 2018 (Dermoscopy)%
} \\

\bottomrule

\end{tabular}

\caption{\textbf{Qualitative segmentation under DP feature release.} Example validation cases (top to bottom: Kvasir-SEG, BUSI, ISIC~2018) at $\varepsilon{=}2$, $K{=}3$. Across all three modalities, CANAL (right column) recovers a substantially more accurate lesion/polyp boundary than uniform noise (third column) at the same privacy budget.}
\label{fig:qualitative}
\end{figure*}

\subsection{Privacy-attack evaluation}
\label{ssec:exp-attacks}

To connect the privacy budget to the concrete threats of \cref{sec:related}, we attack the released features directly. We report membership-inference AUC (closer to $0.5$ is more private) and feature-inversion reconstruction quality (SSIM between the inverted and original image) at $\varepsilon\in\{2,8\}$, for both uniform noise and CANAL. The questions are whether the calibrated noise neutralizes these attacks, and whether CANAL's non-uniform allocation costs any privacy relative to uniform noise at the same budget.

\begin{table}[t]
\centering
\small
\caption{\textbf{Privacy-attack resistance of the released features.}
Membership-inference AUC (closer to $0.5$ is more private) and
feature-inversion SSIM ($\downarrow$).
ISIC, $K{=}3$, top-$10\%$ channels, over $5$ seeds.}
\label{tab:exp-attacks}
\setlength{\tabcolsep}{4pt}
\begin{tabular}{l c c c c}
\toprule
 & \multicolumn{2}{c}{MI-AUC ($\to 0.5$)} & \multicolumn{2}{c}{Inversion SSIM $\downarrow$} \\
\cmidrule(lr){2-3}\cmidrule(lr){4-5}
Mechanism & $\varepsilon{=}2$ & $\varepsilon{=}8$ & $\varepsilon{=}2$ & $\varepsilon{=}8$ \\
\midrule
No noise ($\varepsilon{=}\infty$) & 1.000 & 1.000 & 0.651 & 0.651 \\
\addlinespace[2pt]
Uniform feature noise & 0.500 & 0.494 & 0.039 & 0.030 \\
\textbf{CANAL (ours)} & 0.499 & 0.495 & 0.037 & 0.029 \\
\bottomrule
\end{tabular}
\end{table}

\Cref{tab:exp-attacks} answers both. Without noise ($\varepsilon{=}\infty$) the released features are highly exposed: membership inference is near-perfect (AUC~$1.0$) and feature inversion reconstructs recognizable structure (SSIM~$0.65$). The calibrated noise collapses both attacks: across $\varepsilon\in\{2,8\}$ the membership-inference AUC falls to $0.50$, indistinguishable from the chance level, and the inversion SSIM drops below $0.04$, leaving no usable reconstruction. CANAL and uniform reach essentially the same attack resistance at a given budget, since both satisfy the same $(\varepsilon,\delta)$-DP guarantee. CANAL spends that budget differently across channels rather than spending more of it, so its utility advantage carries no privacy penalty.

\subsection{Ablations}
\label{ssec:exp-ablation}

We ablate the PATE-style ensemble size $K$ on ISIC~2018, using uniform allocation so the effect of $K$ on the per-release noise is isolated from the allocation.

The replace-one sensitivity is $\Delta_c{=}2/K$ (\cref{ssec:accounting}), so a larger ensemble lets a fixed budget buy less noise. Because our benchmarks are single datasets rather than several hospitals' data, the ablation simulates a $K$-site federation by partitioning one dataset into $K$ disjoint cohorts. This holds the total private data fixed and isolates the effect of $K$ on the noise, whereas in a real deployment each of the $K$ sites contributes its own data. \Cref{fig:abl-k} shows the result at $\varepsilon{=}8$: under this fixed-data split, Dice rises steeply with $K$ as the per-release noise shrinks, then flattens beyond $K{=}3$ once the noise floor is low and the smaller per-teacher cohorts offset the gain. We choose $K{=}3$ by default because $K$ is not a free hyperparameter: in the multi-hospital setting it is the number of data-providing sites, set by how many institutions can be recruited and coordinated under data-use agreements rather than tuned. Cross-institutional medical collaborations are typically small~\cite{kaissis2020secure}, so $K{=}3$ is a representative federation size. It is also the regime where channel-aware allocation is most useful: a small ensemble keeps the per-release noise high ($\Delta_c{=}2/K$), and the value of shaping that noise across channels grows with its magnitude. We accordingly report the main results at $K{=}3$.

\begin{figure}[t]
\centering
\includegraphics[width=0.8\linewidth]{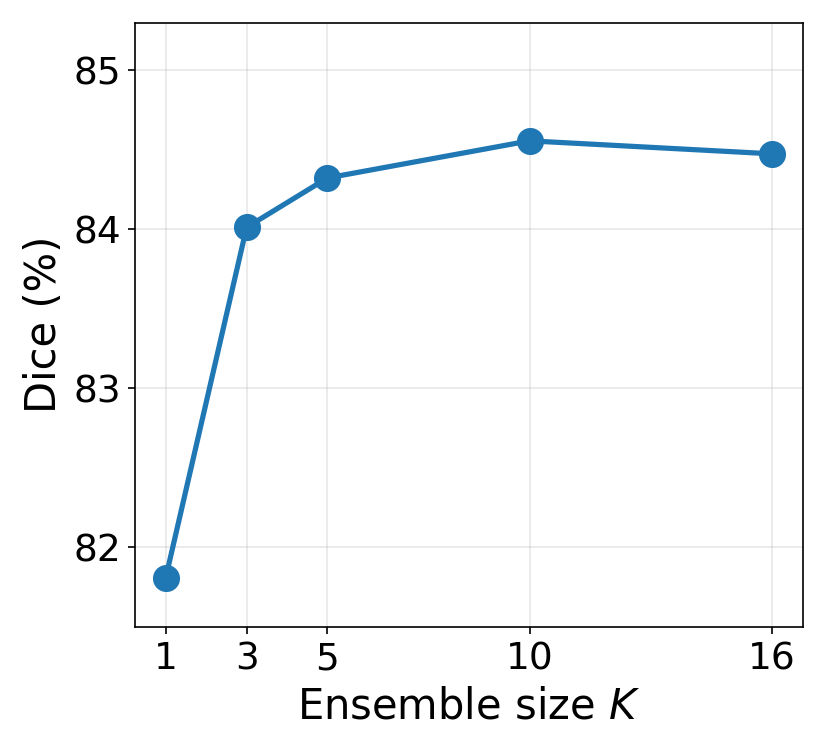}
\vspace{-0.1in}
\caption{\textbf{Ensemble size $K$ vs.\ Dice at $\varepsilon{=}8$} (uniform allocation, mean over $3$ seeds). To isolate the effect of $K$ at a fixed data budget, the ablation partitions one dataset (ISIC~2018) into $K$ disjoint cohorts. In deployment each of the $K$ sites instead contributes its own data. Dice rises steeply through $K{=}3$ and then flattens, so the default $K{=}3$ already captures most of the attainable gain.}
\label{fig:abl-k}
\end{figure}

\section{Conclusion}
\label{sec:conclusion}

We studied how hospitals can share what their models have learned without exposing individual patients, and found that the standard recipe of noising distilled features is both less private and less useful than it needs to be. On the privacy side, re-sampling noise every student iteration releases each patient numerous times, so the privacy cost composes and grows with the number of epochs. A sample-once-per-image release, realized by a single precomputation pass, spends the budget once, so the user-level $\varepsilon$ equals the per-release $\varepsilon$. On the utility side, spreading noise uniformly across feature channels wastes the budget. CANAL, a closed-form channel-importance water-filling allocation, gives task-important channels less noise and provably minimizes importance-weighted distortion at a fixed budget. Because the clipping caps and importance scores are themselves data-dependent, we charge them to the privacy budget through a DP-honest split. Empirically, across dermoscopy, colonoscopy, and ultrasound, the single-pass release neutralizes membership-inference and feature-inversion attacks, while CANAL recovers more task-relevant signal than uniform noise at the same $\varepsilon$.

\paragraph{Limitations and future work.} Several of the margins in \cref{tab:exp-generalize} are small, so the table is best read as directional evidence that channel-aware allocation helps at a fixed budget, under conditions identical to those of the uniform baseline. Our guarantee is user-level under the sample-once model, where each patient contributes one image. Repeated visits or correlated scans per patient would require group-level accounting. We validate on three 2D modalities (dermoscopy, colonoscopy, and ultrasound). Extending the study to 3D volumes and to a larger federation of hospitals with heterogeneous distributions is natural future work. CANAL allocates across channels because channel importance is a cheap, low-sensitivity dataset-level statistic. Extending the allocation to space is honest only for spatially registered data such as atlas-aligned brain MRI, where the importance map becomes a dataset-level prior (\cref{ssec:granularity}).

{\small
\bibliographystyle{ieeenat_fullname}
\bibliography{egbib}
}

\appendix

\section{Pseudocode: sample-once distillation}
\label{app:pseudocode}

Algorithm~\ref{alg:precompute} gives the precompute routine that implements the sample-once-per-image release model of \cref{ssec:release}. It is invoked once before the student training loop. Throughout, $\phi_T$ denotes the (frozen) teacher encoder, $\mathrm{clip}(\cdot;\widehat{c})$ per-channel $L_2$ clipping to caps $\widehat{c}$, and $\eta\sim\mathcal{N}(0,\sigma^2)$ draws independent Gaussian noise per channel and spatial location.

\begin{algorithm}[h]
\caption{Sample-once precompute of noisy teacher bottlenecks, written for a single teacher ($K{=}1$) to keep the notation light. For $K>1$, lines~3--4 are carried out for each of the $K$ teachers and $z_i'$ becomes their per-channel average, after which the noise of line~5 is added once to that average (\cref{ssec:impl}).}
\label{alg:precompute}
\begin{algorithmic}[1]
\Require teacher $\phi_T$; training images $\{x_i\}_{i=1}^{N}$;
         per-channel caps $\widehat{c}\in\mathbb{R}^{C}$;
         per-channel noise scales $\sigma\in\mathbb{R}^{C}$
\Ensure  noisy bottleneck cache
         $\mathcal{C}:\mathrm{img\_id}\!\to\!\mathbb{R}^{C\times H\times W}$
\State $\mathcal{C} \gets \emptyset$
\For{$i = 1,\dots,N$}
  \State $z_i \gets \phi_T(x_i)$
         \Comment{forward through encoder, frozen}
  \State $z_i' \gets \mathrm{clip}(z_i;\widehat{c}) / \widehat{c}$
         \Comment{per-channel norm $\le 1$}
  \State sample $\eta_i \in \mathbb{R}^{C\times H\times W}$ with
         $\eta_i^{(c)} \!\sim\! \mathcal{N}(0,\sigma_c^{2})$
  \State $\widetilde{z}_i \gets (z_i' + \eta_i)\odot\widehat{c}$
         \Comment{denormalize}
  \State $\mathcal{C}[\mathrm{img\_id}(x_i)] \gets \widetilde{z}_i$
\EndFor
\State \Return $\mathcal{C}$
\end{algorithmic}
\end{algorithm}

\noindent During student training, the inner loop's teacher path is replaced by a single look-up $\widetilde{z}_i \leftarrow \mathcal{C}[\mathrm{img\_id}(x_i)]$, eliminating both per-iteration teacher queries and per-iteration noise samples.

\section{Full proof of the channel water-filling theorem}
\label{app:proof}

We restate the importance-weighted optimization problem of \cref{ssec:wf} for convenience:
\begin{equation*}
\min_{\sigma_1,\ldots,\sigma_C>0}\; \sum_{c=1}^{C} s_c\,\sigma_c^{2}
\quad\text{s.t.}\quad
\sum_{c=1}^{C} \frac{\Delta_c^{2}}{2\,\sigma_c^{2}} \le \rho_{\mathrm{rel}}.
\end{equation*}

\paragraph{Change of variables.} Let $u_c = 1/\sigma_c^{2}$, so $u_c > 0$ and $\sigma_c^{2} = 1/u_c$. The problem becomes
\begin{equation}
\min_{u_1,\ldots,u_C>0}\; \sum_{c=1}^{C} \frac{s_c}{u_c}
\quad\text{s.t.}\quad
\sum_{c=1}^{C} \frac{\Delta_c^{2}}{2}\,u_c \le \rho_{\mathrm{rel}}.
\label{eq:dual}
\end{equation}
Both the objective $\sum_c s_c/u_c$ (convex in $u_c$ on the positive orthant) and the linear constraint give a convex program with a unique minimizer. The constraint is active at the optimum because the objective is strictly decreasing in each $u_c$.

\paragraph{Lagrangian.} Form the Lagrangian with multiplier $\lambda > 0$:
\begin{equation*}
\mathcal{L}(u,\lambda)
 \;=\; \sum_{c=1}^{C} \frac{s_c}{u_c}
       + \lambda\left(\sum_{c=1}^{C}\frac{\Delta_c^{2}}{2}u_c - \rho_{\mathrm{rel}}\right).
\end{equation*}
Stationarity $\partial\mathcal{L}/\partial u_c = 0$ gives
\begin{equation*}
-\frac{s_c}{u_c^{2}} + \frac{\lambda\,\Delta_c^{2}}{2} = 0
\;\;\Longrightarrow\;\;
u_c^{\star} \;=\; \sqrt{\frac{2\,s_c}{\lambda\,\Delta_c^{2}}}
\;=\; \sqrt{\frac{2}{\lambda}}\;\frac{\sqrt{s_c}}{\Delta_c}.
\end{equation*}

\paragraph{Solving for $\lambda$ via the active constraint.} Substitute back into $\sum_c \tfrac{\Delta_c^{2}}{2}u_c^{\star} = \rho_{\mathrm{rel}}$:
\begin{equation*}
\sum_{c=1}^{C}\frac{\Delta_c^{2}}{2}\sqrt{\frac{2}{\lambda}}\,\frac{\sqrt{s_c}}{\Delta_c}
 \;=\; \rho_{\mathrm{rel}}
\;\;\Longrightarrow\;\;
\sqrt{\frac{1}{2\lambda}}\,\sum_{c=1}^{C}\Delta_c\sqrt{s_c}
 \;=\; \rho_{\mathrm{rel}}.
\end{equation*}
Solving for $\lambda$:
\begin{equation*}
\lambda \;=\; \frac{1}{2\,\rho_{\mathrm{rel}}^{2}}
              \Bigl(\sum_{c=1}^{C}\Delta_c\sqrt{s_c}\Bigr)^{2}.
\end{equation*}

\paragraph{Optimal $\sigma_c^{\star}$.} Substituting $\lambda$ back into $u_c^{\star}$ and using $\sigma_c^{\star} = 1/\sqrt{u_c^{\star}}$:
\begin{equation*}
\sigma_c^{\star}
 \;=\;
 \frac{1}{(2/\lambda)^{1/4}}\,
 \frac{\Delta_c^{1/2}}{s_c^{1/4}}
 \;=\;
 \underbrace{\sqrt{\frac{1}{2\,\rho_{\mathrm{rel}}}\sum_{c=1}^{C}\Delta_c\sqrt{s_c}}}_{\displaystyle\kappa}
 \;\frac{\sqrt{\Delta_c}}{s_c^{1/4}},
\end{equation*}
which is exactly the channel water-filling allocation of \cref{thm:wf}.

\paragraph{Optimal objective and strict improvement.} The optimal value of the original objective is
\begin{equation}
\begin{aligned}
J_{\mathrm{WF}}^{\star}
 &\;=\; \sum_{c=1}^{C} s_c\,(\sigma_c^{\star})^{2}
  \;=\; \kappa^{2}\sum_{c=1}^{C}\Delta_c\sqrt{s_c}\\
 &\;=\; \frac{1}{2\,\rho_{\mathrm{rel}}}
       \Bigl(\sum_{c=1}^{C}\Delta_c\sqrt{s_c}\Bigr)^{2}.
\end{aligned}
\label{eq:wf-obj}
\end{equation}
For uniform $\sigma_{\mathrm{unif}}$ saturating the same budget, $\sigma_{\mathrm{unif}}^{2} = (\sum_c \Delta_c^{2})/(2\,\rho_{\mathrm{rel}})$, so the uniform-allocation objective is
\begin{equation}
J_{\mathrm{unif}}
 \;=\; \sigma_{\mathrm{unif}}^{2}\sum_{c}s_c
 \;=\; \frac{1}{2\,\rho_{\mathrm{rel}}}
       \Bigl(\sum_{c}\Delta_c^{2}\Bigr)\Bigl(\sum_{c}s_c\Bigr).
\label{eq:unif-obj}
\end{equation}
Comparing \cref{eq:wf-obj} and \cref{eq:unif-obj} via Cauchy--Schwarz applied to the vectors $(\Delta_c)$ and $(\sqrt{s_c})$, or directly via $(\sum_c \Delta_c\sqrt{s_c})^{2} \le (\sum_c \Delta_c^{2})(\sum_c s_c)$ with equality iff $s_c$ is constant, we obtain
\begin{equation*}
J_{\mathrm{WF}}^{\star} \;\le\; J_{\mathrm{unif}},
\end{equation*}
with strict inequality whenever the $\{s_c\}$ are not all equal. This establishes the strict-improvement claim of \cref{thm:wf} and completes the proof. \qed

\end{document}